\documentclass{article}

  \usepackage[preprint,nonatbib]{neurips_2020}
  \newcommand{\co}[1]{}
\usepackage{amsthm}
\usepackage{subfigure}
\usepackage{epsfig}
\usepackage{amsfonts,amsmath,amscd, amssymb}
\newtheorem{theorem}{Theorem}
\usepackage[utf8]{inputenc} % allow utf-8 input
\usepackage[T1]{fontenc}    % use 8-bit T1 fonts
\usepackage{hyperref}       % hyperlinks
\usepackage{url}            % simple URL typesetting
\usepackage{booktabs}       % professional-quality tables
\usepackage{amsfonts}       % blackboard math symbols
\usepackage{nicefrac}       % compact symbols for 1/2, etc.
\usepackage{microtype}      % microtypography

\title{Nystr$\ddot{\textrm{o}}$m Approximation with Nonnegative Matrix Factorization}

\author{
Yongquan Fu\\
yongquanf@nudt.edu.cn\\
National University of Defense Technology
}

\begin{document}

\maketitle

\begin{abstract}
Motivated by the needs of estimating the proximity clustering with partial distance measurements from vantage points or landmarks for remote networked systems, we show that the proximity clustering problem can be effectively formulated as the Nystr$\ddot{\textrm{o}}$m approximation problem, which  solves the kernel K-means clustering problem in the complex space. We implement the Nystr$\ddot{\textrm{o}}$m approximation based on a landmark based Nonnegative Matrix Factorization (NMF) process. Evaluation results show that the proposed method finds nearly optimal clustering quality  on both synthetic and real-world data sets  as we vary the range of parameter choices and network conditions.
\end{abstract}

\section{Introduction}

This paper is motivated to determine the clustering structure of global DNS servers without direct control. The Domain Name System (DNS) system \cite{DNS11} is one of the most important components of the Internet infrastructure, which converts a domain or host name to one or multiple IP addresses on the Internet so that clients can establish the network connection towards the service provider. It is estimated that there are tens of thousands of public and open DNS servers on the Internet, yet we still have little knowledge  on the locations or the distance between DNS servers, as  we cannot directly collect network distances between DNS servers. Suppose that we know the clustering structure of  DNS servers, we can have a powerful unsupervised learning model for diverse situation-awareness tasks.   

Clustering of networked systems based on network distances  provides a compact summary representation of the global situation awareness.  Traditionally, clustering analysis partitions a set of data items to a number of clusters with the features in the data items. Each data item is typically represented with a vector in a coordinate space, where the vector distance of two data items determines the possibility of clustering them to the same cluster.  It is challenging to identify the distance based clustering structures  in real time, since the servers do not allow for installing any third-party measurement software. Thus we can only observe the latency towards the DNS servers. As a result, the network distance matrix between  internal servers are unknown to external observers. Moreover, measuring the network distance is also costly due to the synchronization of measurements. Due to the hardness of the problem,  prior researchers typically choose clustering heuristics without theoretical guarantees \cite{DBLP:journals/pami/RothLKB03,DBLP:conf/icdcs/SuCBK08,DBLP:journals/sigmetrics/ChenLKO02,DBLP:journals/ccr/SharmaXBL06}. Recently, Wang et al. \cite{DBLP:journals/corr/WangGM17a} provide improved approximation bounds on the symmetric positive semi-definite (SPSD) matrix for the kernel K-means clustering problem based on the Nystr$\ddot{\textrm{o}}$m approximation over the feature space. Given a SPSD matrix $\mathrm{K} \in R^{n \times n}$  and  a sketching matrix $\mathrm{P} \in R^{n \times r}$, the Nystr$\ddot{\textrm{o}}$m method derives $\mathrm{C} = \mathrm{K}  \mathrm{P}$ and $\mathrm{D} = \mathrm{P}^T \mathrm{K}  P$, and then approximates $\mathrm{K}$ with $\mathrm{C} \mathrm{D} ^{\dagger} \mathrm{C}^T$, where  $\mathrm{D} ^{\dagger}$ represents $\mathrm{D}$'s  Moore-Penrose inverse.  It is natural to ask whether we could relax the SPSD property on generalized distance matrices?

We formulate a Nystr$\ddot{\textrm{o}}$m approximation framework for generalized complex-space kernel matrix that arises from the network distance matrix.  For the proximity clustering on the networked system, we collect a small number of measurements between vantage-point servers and target servers, and try to find the proximity clustering structure for DNS servers based on  partially observed measurements.   

 We show that the Nystr$\ddot{\textrm{o}}$m approximation with NMF 
based on the symmetric distance matrix, is equivalent to the kernel
K-means clustering on a complex-space kernel matrix. The clustering interpretation based the NMF has been extensively studied \cite{DBLP:conf/sdm/DingH05,DBLP:conf/kdd/DingLPP06}.   Proving the clustering interpretations for generalized complex-space kernel matrix is still an open problem. We show that  one NMF factor matrix has identical clustering results with the optimal kernel K-means clustering
indicator matrix; and the other factor matrix reflects the clustering validity of the whole set of nodes.  

We test the clustering quality with  synthetic and real-world network distance matrices. Experimental results show that our approach is efficient and can find stably accurate clustering structure, with variable-sized data sets that may contain missing items, and adapts to topology changes and link dynamics.

%We collect a small number of measurements between vantage-point servers and target servers, and try to find the proximity clustering structure for DNS servers based on  partially observed measurements.  

\section{Problem Formulation}

To efficiently test the network distance based clustering, a concise network distance model that precisely preserves the inter-node distances is required, such that the optimal clustering
structure obtained from the network distance  model, is identical
to the hidden clustering structure of the original pairwise distance
matrix. 

Generally, the pairwise network distances between $N$ network hosts called nodes, can be abstractly represented as a $N$-by-$N$ matrix $W$.  We assume the network distance matrix to be symmetric, since the clustering input needs to be a metric. For example, the Round Trip Time (RTT) satisfies this  assumption. For other metric such as hops, we may consider the sum of the metrics in the  forward path and the reverse path for a node pair.  
 
 Let $W(x,y)$ denote the pairwise distance from node $x$ to node $y$.  Given $N$ nodes, let the distance mapping in $S$ as $W: {N  \times N}  \to {R}^{+}$, which satisfies: (a) $W\left( {x,x} \right) = 0$ ; (b) $W\left( {x,y} \right) = W\left( {y,x} \right)$, $\forall x,y \in N$. 
 
The SVD of the distance matrix $W$ can be represented as $W = U Q V^T$, where $U \in R^{N \times r}$ and $V \in R^{N \times r}$ are orthogonal matrices, and $Q = U^T W V \in R^{N \times N}$ is diagonal with $Q = {diag} (\delta _1, \delta _2, \dots, \delta _r)$ with the non-negative numbers  $\delta _1 \ge \delta _2  \ge \dots  \ge\delta _r  \ge 0$ being the singular values of $W$.  The dimensionality $r$ of the network distance matrix is known to be approximately low, since the wide-area routing paths usually share some path segments, yielding correlations among different node pairs \cite{Abrahao2008IDS14525201452541,tonLeeZSS10}.

Further, let ${X^{h}}$ denote the complex conjugate operator of a hermitian matrix $X$, let the eigenvalue decomposition of $W$ be  $W = Z D Z^{h}$, where $Z$ is an orthogonal matrix, and the columns of $V$ are eigenvectors for $W$ and $D$ is a diagonal matrix with diagonal entries $\lambda _1, \lambda _2, \dots, \lambda _r$ being the eigenvalues of $W$.  Due to the symmetry of the $W$, we see that the SVD of the distance matrix $W$ is equivalent to the eigenvalue decomposition of $W$, where the singular values serve as the magnitudes of the eigenvalues $U = V$ and $\delta _i = \| \lambda _i \|$. 

We define a projected operator for $W$ based on the eigenvalue decomposition as:  $\phi$: $\mathbf{x} = \mathbf{x} D^{\frac{1}{2}}$, where $D^{\frac{1}{2}}$ is a diagonal hermitian matrix, since some entries in $D^{\frac{1}{2}}$ may be complex numbers due to the negativity of some eigenvalues. Next, we can see that the distance matrix $W$ can be represented with the projected operators:
\begin{equation}
W=\phi\left(Z\right)\phi\left(Z\right)^{h}
\end{equation}

Since $\phi\left(Z\right)\phi\left(Z\right)^{h}$ completely preserves
the inter-node distances, we adopt the network distance  model as
$W=\phi\left(Z\right)\phi\left(Z\right)^{h}$. Now $\phi\left(Z\right)$
can be regarded as generalized vectors of nodes.

Based on the kernel matrix representation, we select the K-means clustering as the compact clustering objective, which is one of the most popular clustering methods with rich theoretical extensions \cite{DBLP:conf/sdm/DingH05,DBLP:conf/kdd/DingLPP06,DBLP:journals/tkde/WuLXCC15,DBLP:journals/tkde/WangWSXSL15,DBLP:journals/tkde/LiuWLTF17}. If a clear-separation
clustering structure is identified, then the clustering result is close to the optimal K-means clustering \cite{DBLP:conf/icml/Meila06}.

Let the complex-number vectors be   represented
as $\{{\mathbf{x}}_{1},\ldots,{\mathbf{x}}_{n}\}$, and data items are to be divided
into $K$ groups, denoted as $C_{1},\ldots,C_{K}$. Let $W$ be the symmetric pairwise distance matrix
between $n$ nodes.  Let $K $ denote the number of clusters, $ N_i $  the number of items in the $i$-th cluster, $c_c $   the set of items in the $i$-th cluster, $H \in {\{ 0,1\} ^{N \times K}} $   the clustering indicator matrix for the items,  and $H(i,j) = 1$ the indicator that  the item $i$ is in the $j$-th cluster, and  ${{{\tilde c}}_k} = \frac{1}{{{N_k}}}\sum\nolimits_{p \in {C_k}} {{{\mathbf{x}}_p}} $ the centroid of the $K $-th cluster

Our goal is to design clustering methods   with solid foundations and monitor global clustering validity, i.e., whether there are significant separations between different clusters.  The kernel K-means clustering generalizes the K-means clustering,  by mapping data items into a high dimensional feature space, i.e.,
$x_{i}\rightarrow\phi(x_{i})$. Kernel K-means clustering can efficiently
find the clustering structure with nonlinear separations in the space
\cite{DBLP:conf/sdm/DingH05}. It defines a kernel function $\phi\left( {{\mathbf{x}}_i} \right)$ on each item' coordinate ${{\mathbf{x}}_i}$, which yields the kernel version of the optimization objective:
\begin{equation}\label{KernelMeans}
  {\mathrm{min}}{\mkern 1mu} {J}\left( {H,{{\tilde \mu }_k}} \right) = \sum\limits_{k = 1}^K  {\sum\limits_{i = 1}^N {{H_{ik}}{{\left\| {\phi \left( {{{{\mathbf{x}}_i}}} \right) - {{\tilde \mu }_k}} \right\|}^2}} }
\end{equation}
where ${{\tilde \mu }_k} = \frac{1}{{{N_k}}}\sum\nolimits_{i \in {C_k}} {\phi \left( {{{{\mathbf{x}}_i}}} \right)} $ denotes the centroid of the $K$-th cluster. The kernel K-means clustering method is more suitable for identifying non-linear structures and especially works well for real-life data sets containing noises.

\section{Nystr$\ddot{\textrm{o}}$m Approximation}

We present a Nystr$\ddot{\textrm{o}}$m approximation method HSH based on the clustering interpretations of the NMF.

First, we randomly select a subset of nodes as landmark nodes and let each landmark node independently measure the distances towards the other landmarks, and send the probed distance vector to a centralized master.  Specifically, we obtain the pairwise RTT matrix between landmarks, and the RTT from landmarks to target servers (non-landmark for short). The probed results are two parts:  (i) $W_{L\times L}$ be the distance matrix of the landmarks; (ii) let $W_{D\times L}$ be the distance matrix from non-landmarks to landmark nodes. We can see that the RTT values between target servers are unobservable.

Second, the master computes the  NMF results based on the collected pairwise distance matrix between landmark nodes. 
We seek to optimize
 \begin{equation}\label{EqRelaxed}
\underset{{H}\geq0,S\geq0}{min}\Vert W-{H}S{H}^{T}\Vert^{2}
\end{equation} 
with multiplicative updating rules as follows to minimize Eq.~\eqref{EqRelaxed}:
\[
{H}_{jk}\leftarrow{H}_{jk}\sqrt{\frac{\left(W_{L\times L}{H}S\right)_{jk}}{\left({H}{H}^{T}W_{L\times L}{H}S\right)_{jk}}}, S_{kl}\leftarrow S_{kl}\sqrt{\frac{\left({H}^{T}W_{L\times L}{H}\right)_{kl}}{\left({H}^{T}{H}S{H}^{T}{H}\right)_{kl}}}.
\]
for landmark $j$ and each dimension $k \in \left[ 1, K \right] $and for each row $k \in \left[ 1, K\right] $ and each column $l \in \left[ 1, K\right] $. 

Generally,  the diagonal elements of the matrix
$S$ refer to sums of intra-cluster distances, while the off-diagonal
elements $\left(i,j\right)$ of the matrix \emph{S} correspond to sums of distances between nodes in the $i$th and the $j$th
clusters, for $i\neq j$. Therefore, if there is a distinct gap between
the diagonal element of the $i$th row vector and the off-diagonal
element $\left(i,j\right)$ of matrix $S$, for $j\neq i$, the separation between the $i$th cluster and the $j$th cluster is apparent; otherwise, these two clusters are likely to overlap each other.

Third, afterwards, the master computes the NMF results for target servers based on factor matrices of landmarks as well as the network distances from the landmarks to target servers. Specifically, the master  optimizes a least square unconstrained optimization
problem: 
\begin{equation}
{H}_{i}=\underset{P}{min}\Vert W_{iL}-P\cdot S\cdot{H}_{L}^{T}\Vert^{2}
\label{eq:leastSquare}\end{equation}
Which has a closed-form global optimal value as:
 \[
{H}_{i}=W_{iL}S{H}_{L}^{T}\left(\left(S{H}_{L}^{T}\right)^{T}\left(S{H}_{L}^{T}\right)\right)^{-1}
\].

\subsection{Approximation Analysis}

The above process can be summarized as follows:
\[
\underset{{H}_{L}\geq0,S\geq0}{min}\underset{\textrm{Stage 1}}{\underbrace{\Vert W_{L\times L}-{H}_{L}S{H}_{L}^{T}\Vert^{2}}}+\underset{\textrm{Stage 2}}{\underbrace{\Vert W_{D\times L}-{H}_{D}S{H}_{L}^{T}\Vert^{2}}}, 
\]

which leads to the following matrix
approximation goal:

\begin{equation}
perm\left(W\right)  =\left[\begin{array}{cc}
W_{L\times L} & W_{D\times L}^{T}\\
W_{D\times L} & W_{H\times H}\end{array}\right] 
\approx \left[\begin{array}{cc}
{H}_{L}S{H}_{L}^{T} & {H}_{L}S{H}_{D}^{T}\\
{H}_{D}S{H}_{L}^{T} & {H}_{D}S{H}_{D}^{T}\end{array}\right]=\hat{H}S\hat{H}^{T}
\end{equation}

where $perm\left(W\right)$ represents the reordered matrix of $W$
according to the index sequence of nodes (landmarks, non-landmarks),
and $\hat{H}$ is the transpose of $\left[\begin{array}{cc}
{H}_{L} & {H}_{D}\end{array}\right]$. Then, the approximation error of target servers (non-landmarks) can be described as $\Vert W_{H\times H}-{H}_{D}S{H}_{D}^{T} \Vert^{2}$.

\subsection{Clustering Analysis}

We next show that with the complex vector representation  $\phi\left(Z\right)$, the kernel K-means clustering based on the network distance  model $W=\phi\left(Z\right)\phi\left(Z\right)^{h}$,
is equivalent to the NMF on the global distance matrix.

\begin{theorem}
\label{thm:1} 
Let $W=\phi\left(Z\right)\phi\left(Z\right)^{h}$, the kernel K-means clustering
objective on $Q$: \begin{equation}
J_{K}=\underset{H}{min}\Sigma_{k=1}^{K}\Sigma_{i=1}^{n}H_{ik}\Vert\phi\left(Z_{i}\right)-\overline{c_{k}}\Vert^{2}\label{eq:Kernel K}\end{equation}
is equivalent to the NMF (\ref{eq:HSH}) defined as:
\begin{equation}
\underset{{H}\geq0,S\geq0}{min}\Vert W-{H}S{H}^{T}\Vert^{2},s.t.,{H}^{T}{H}=I,S\:\textrm{is diagonal},\label{eq:HSH}
\end{equation}
where ${H}=\left(h_{1},\ldots,h_{K}\right)\in R_{+}^{n\times K}$,
$S\in R_{+}^{K\times K}$, $R_{+}$ represents the set of nonnegative matrices.   
\end{theorem}
\begin{proof}
 Let the clustering indicator matrix $H=\{0,1\}^{n\times K}$, based on \cite{DBLP:conf/sdm/DingH05}, the kernel K-means clustering objective is equivalent to 

\begin{equation}
J_{K}=\min  tr\left(\phi\left(x\right)\phi\left(x\right)^{h}\right)-tr\left(H^{T}\phi\left(x\right)\phi\left(x\right)^{h}H\right)\label{eq:J}\end{equation}

The first item of (\ref{eq:J}) is constant, with the symmetric
network distance  model $W=\phi\left(Z\right)\phi\left(Z\right)^{h}$,
the optimization of (\ref{eq:J}) is equivalent to

\begin{equation}
J_{W}=\underset{H,H\geq0}{max}tr\left(H^{T}WH\right)\label{eq:J_d}\end{equation}

The choices of the items of $H$ are either 1 or 0. Since it is
hard to complete an integer optimization problem, we relax the integer
constraint of the matrix $H$: because each node belongs to only one
cluster, there is only one nonzero item in each row vector of $H$,
which can be described as: (1) $\left(H^{T}H\right)_{ij}=0,i\neq j$;
(2) $\left(H^{T}H\right)_{ii}=|C_{i}|=n_{i}$, which is the number
of nodes in the $i$th cluster.

Let $S=diag\left(H^{T}H\right)=diag\left(n_{1},\ldots,n_{K}\right)\in R^{K\times K}$,
thus $H^{T}H=S$. The objective of (\ref{eq:J_d}) is equivalent to

\begin{equation}
J_{S}=\underset{H^{T}H=S,H\geq0}{max}tr\left(H^{T}WH\right)\label{eq:J_m}\end{equation}

Now the choices of the items of $H$ are mapped to a continuous range.
However, note that $S$ is an unknown matrix, due to the fact that
the clustering information is unknowable in advance. To eliminate
$S$\emph{ }in the restraints of (\ref{eq:J_m}), let $\tilde{H}=H\left(H^{T}H\right)^{-\frac{1}{2}}$,then
\[
\tilde{H}^{h}\tilde{H}=H\left(H^{T}H\right)^{-1}H=I
\]
and
\[
\tilde{H}S\tilde{H}^{h}=H\left(H^{T}H\right)^{-\frac{1}{2}}\left(H^{T}H\right)\left(H^{T}H\right)^{-\frac{1}{2}}H^{T}=HH^{T}.\]
Now the optimal clustering index of each node is equal to
the column number of the nonzero element of each row vector of ${H}$. 

The optimization of (\ref{eq:J_m}) is equivalent to 
\begin{equation}
\underset{\tilde{H}^{h}\tilde{H}=I,\tilde{H},S\geq0}{max}tr\left(\left(\tilde{H}S^{\frac{1}{2}}\right)^{h}W\left(\tilde{H}S^{\frac{1}{2}}\right)\right),S\textrm{ is diagonal}\label{eq:J_max}
\end{equation}
and
 \begin{equation}
\Vert\tilde{H}S\tilde{H}^{h}\Vert^{2}=\Vert HH^{T}\Vert^{2}=tr\left(HH^{T}HH^{T}\right)
=tr\left(H^{T}HH^{T}H\right)=tr\left(SS\right)=\Sigma_{i=1}^{K}n_{i}^{2},
\end{equation}
which is a constant, the optimal values of ${H}$ and $S$ in
(\ref{eq:J_max}) are solutions to the following objective:
\begin{equation}
\underset{\tilde{H}^{h}\tilde{H}=I,\tilde{H},S\geq0}{min}\Vert W\Vert^{2}-2tr\left(\left(\tilde{H}S^{\frac{1}{2}}\right)^{h}W\left(\tilde{H}S^{\frac{1}{2}}\right)\right)
+\Vert\tilde{H}S\tilde{H}^{h}\Vert^{2}
=\underset{\tilde{H}^{h}\tilde{H}=I,\tilde{H},S\geq0}{min}\Vert W-\tilde{H}S\tilde{H}^{h}\Vert^{2}
\end{equation}

Now keeping the orthogonal constraint of ${H}$ and the diagonal
constraint of $S$, the optimal matrix ${H}$ equivalently corresponds
to the clustering indicator matrix $H$ in the kernel K-means clustering
objective.
\end{proof}

The clustering interpretation of the NMF generalizes to the complex kernel matrix:  The number $K$ corresponds to the total number of clusters. The factor matrix ${H}$ represents a clustering indicator
matrix, which   indicates the index of the cluster for each node. The diagonal elements of the matrix $S$ represents the sizes of the corresponding
clusters; while the off-diagonal items of the matrix $S$ stand for the magnitude of inter-cluster distances. Therefore, there are large gaps between the diagonal items and the off-diagonal items that are in the same rows. Thus the matrix $S$ indicates the global clustering validity. 

Next, we show that the NMF on the network distance matrix over landmarks has a close connection with the coresets based K-means clustering methods \cite{DBLP:conf/compgeom/FeldmanMS07,DBLP:conf/soda/FeldmanSS13,DBLP:conf/stoc/Har-PeledM04,DBLP:journals/siamcomp/Chen09}.

\textbf{Definition}: Suppose that the landmarks satisfy the coreset property, such that let $Q$ be a set of data points and $\epsilon > 0$, let $cost (x,B) = \min _{y \in B} \| x,y \|$ for $B \in Q$,  let $cost(A, B) = \sum _{x, x \in A } cost (x,B)$, and the  set of landmarks  $Q_l \in Q$ are called an $\epsilon$-coreset, if for every set of cluster centers $C$, we have $(1- \epsilon) \cdot cost(Q, C) \le cost(Q_l, C) \le (1 +  \epsilon) \cdot cost(Q, C)$. 

Har{-}Peled and Mazumdar showed that  the optimal clustering result on the coreset  is also bounded by at most $(1 +  \epsilon)$ times the optimal clustering result on the whole set of data points \cite{DBLP:journals/corr/abs-1810-12826}.  As Theorem \ref{thm:1} shows the equivalence between the NMF and the kernel K-means clustering,  the optimal clustering by the NMF on the network distance matrix over the landmarks is bounded by $(1 +  \epsilon)$ times  the optimal clustering result by the NMF on the global network distance matrix over all data points. In other words, the optimal factor matrices of the landmarks serve as $(1 +  \epsilon)$-approximation for target servers. As a result, the optimal solutions of the proposed  HSH method also yields an $(1 +  \epsilon)$-approximation for the kernel K-means clustering objective on the complex-space kernel matrix.

\section{Evaluation}
\label{EvaSec}

In this section, we validate whether HSH can find real clustering
structures, and verify the clustering quality. 

\subsection{Evaluation Setup}

In the network distance  matrix, there are no ground-truth clustering results of the decentralized
nodes in advance, we evaluate the clustering quality based on two metrics: (i) Silhouette Coefficient: For each node $p_{i}$,
first, the averaged distance (denoted as $a_{i}$) between node $p_{i}$
and the nodes in the same cluster are computed; second, the averaged
distance (denoted as $b_{i}$) between $p_{i}$ and the nodes in different
clusters are computed, then the silhouette coefficient of $p_{i}$
is $\frac{b_{i}-a_{i}}{max(a_{i},b_{i})}$. The silhouette
coefficient varies between -1 and 1, if it approaches -1, the clustering
effectiveness of node $p_{i}$ is insignificant; otherwise, the clustering
effectiveness of node $p_{i}$ is much higher as silhouette coefficient
approaches one. (ii) Gain Ratio. It quantifies the averaged
ratios of the distance reductions by communicating with nodes in the same clusters. The gain ratio of any node $p_{i}$ is defined by the ratio of the mean inter-cluster distance $b_{i}$ to the mean intra-cluster
distance $a_{i}$ , where $b_{i}$ and $a_{i}$ are identical  those in the definition of the silhouette coefficient above. 

We choose both synthetic and real-world network distance data sets for studying the performance of the clustering process: (1)Synthetic. The data set is provided by the Matlab software which is originally used for testing K-means clusterings. In a 4-dimensional Euclidean space, 560 data items are generated, which consist of four clustering centroids and are classified into four groups. (2) Static Data Sets. (i) \textbf{DNS1143}, a symmetric RTT matrix between 1143 DNS servers by the MIT P2PSim project \cite{P2PSimData}  with the King method \cite{Gummadi2002_637203}. (ii) \textbf{DNS3997}, a symmetric delay matrix collected between 3997 DNS name servers by Zhang \textit{etal.}~\cite{DBLP:journals/ton/ZhangNNRDW10} using the King method.   (3) Dynamic Data Set. This data set was collected in summer 2014 for three hours between 99 wide-area servers and mobile nodes \cite{DBLP:conf/sigcse/CapposBKA09} . Each interval aggregates pairwise RTT samples within 15.7 seconds, which indicates short-term dynamics.

\subsection{Synthetic Data Set}

The dimensions of factor matrices can be uniquely determined as the clustering number. To validate whether HSH can find accurate clustering structures,
we use the Synthetic data set, since the ground-truth clustering are computed with K-means clustering algorithm in Matlab configured with 20 random repetitions (denoted as Origin). Then, we compute clustering results based on HSH and the centralized NMF (denoted as Centralized). For HSH, 25 nodes are selected as landmarks uniformly at random, and the dimension of the factor matrices for HSH and that of Centralized are both set to 4. 

Figure \ref{fig:validation} show clustering results. The clustering
quality of HSH is significantly better than the centralized NMF, and is approximately the same as that of the ground-truth clustering
results. Therefore, the two-phase matrix factorization process of HSH can efficiently preserve the optimal K-means clustering structure. 

The clustering quality of centralized NMF, is much lower than that of HSH. Since the matrix factorization process is easily caught in abundant poor-performance local minima, by directly  operating on the complete distance matrix. On the contrary, for
HSH, the size of the matrix factorization problem is reduced, by
selecting only a small subset of landmarks to carry out the multiplicative
update procedures, thus more efficient solutions can be found.

Besides, the running time of centralized NMF (938.18 s) is much longer than that of
HSH (only 0.39 s). In summary, HSH can efficiently find accurate clustering results when clustering structures have clear separations, as confirmed by the Synthetic data set.

%For HSH and the ground-truth clustering, over 80\% of nodes have silhouette coefficients larger than 0.6. 

%
\begin{figure}[!t]
\centering
 \subfigure[Silhouette coefficient]{%
    \label{fig:ex3-a}%
    \epsfig{file=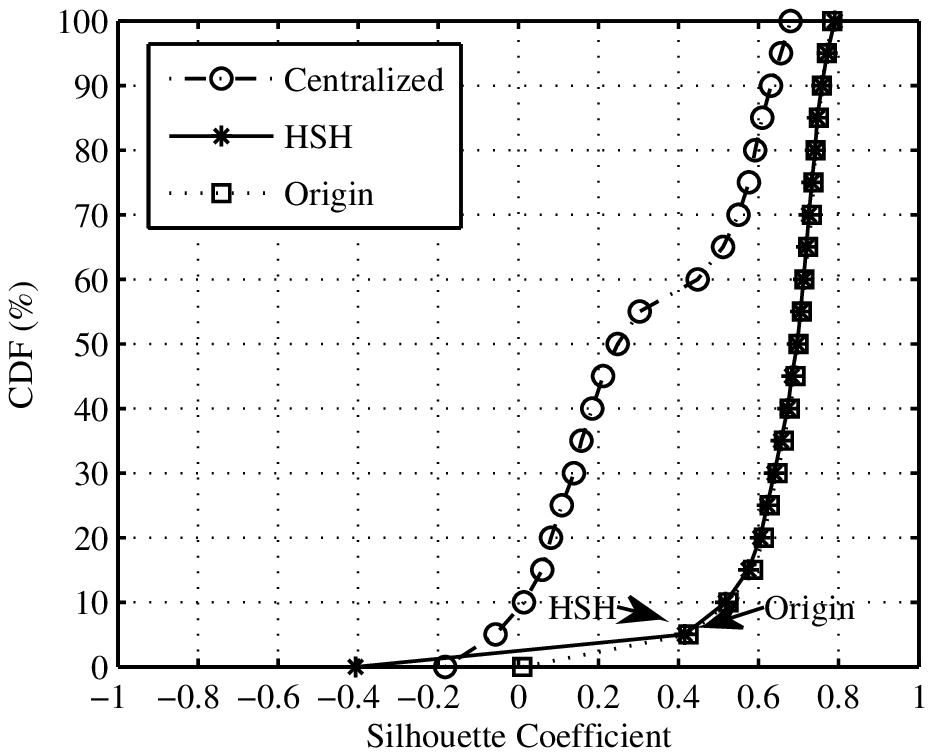,width=0.3\hsize}}
 % %\hspace{8pt}%
  \subfigure[Gain ratio]{%
    \label{fig:ex3-b}%
    \epsfig{file=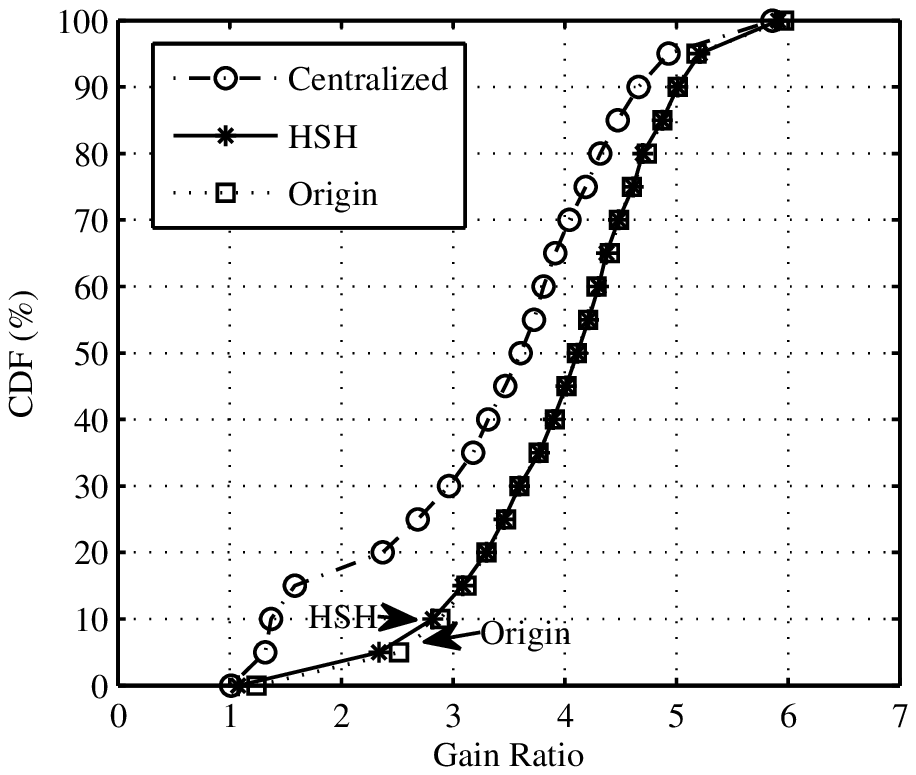,width=0.3\hsize}}

\caption[Clustering validation over Synthetic data set]{\label{fig:validation} Cumulative Distribution Function (CDF) of silhouette coefficients.}
\end{figure}

\subsection{Pairwise Network Distance Matrix Dataset}

Next, we compare the clustering quality of HSH with several optimized clustering methods on the DNS dataset, include the centralized NMF (denoted as centralized), the K-means clustering over the dimension-reduced vectors by the Singular Value Decomposition (SVD), and a K-means clustering over the network coordinates computed by vivaldi, one of the most popular methods \cite{Dabek04vivaldi:a}. We run Centralized and SVD methods with the complete RTT matrix. We set the same number of landmark nodes for HSH and vivaldi.

\subsubsection{Landmark Number}

We fix the number of clusters to three and change the number of landmarks from 20 to 40. Figure \ref{fig:VLS} and \ref{fig:VLG} shows the variations of the silhouette coefficients and those of the gain ratios of four methods, respectively. We can see that increasing the number of landmarks generally improves the clustering accuracy for HSH and vivaldi, especially for the set of poorly clustered nodes,  since both methods become more robust with increasing observations.  

%Therefore, we need to select a proper number of landmarks to obtain good clustering stability. 

 \begin{figure*}
\centering
 \subfigure[DNS1143]{%
    \label{fig:ex4-a}%
    \epsfig{file=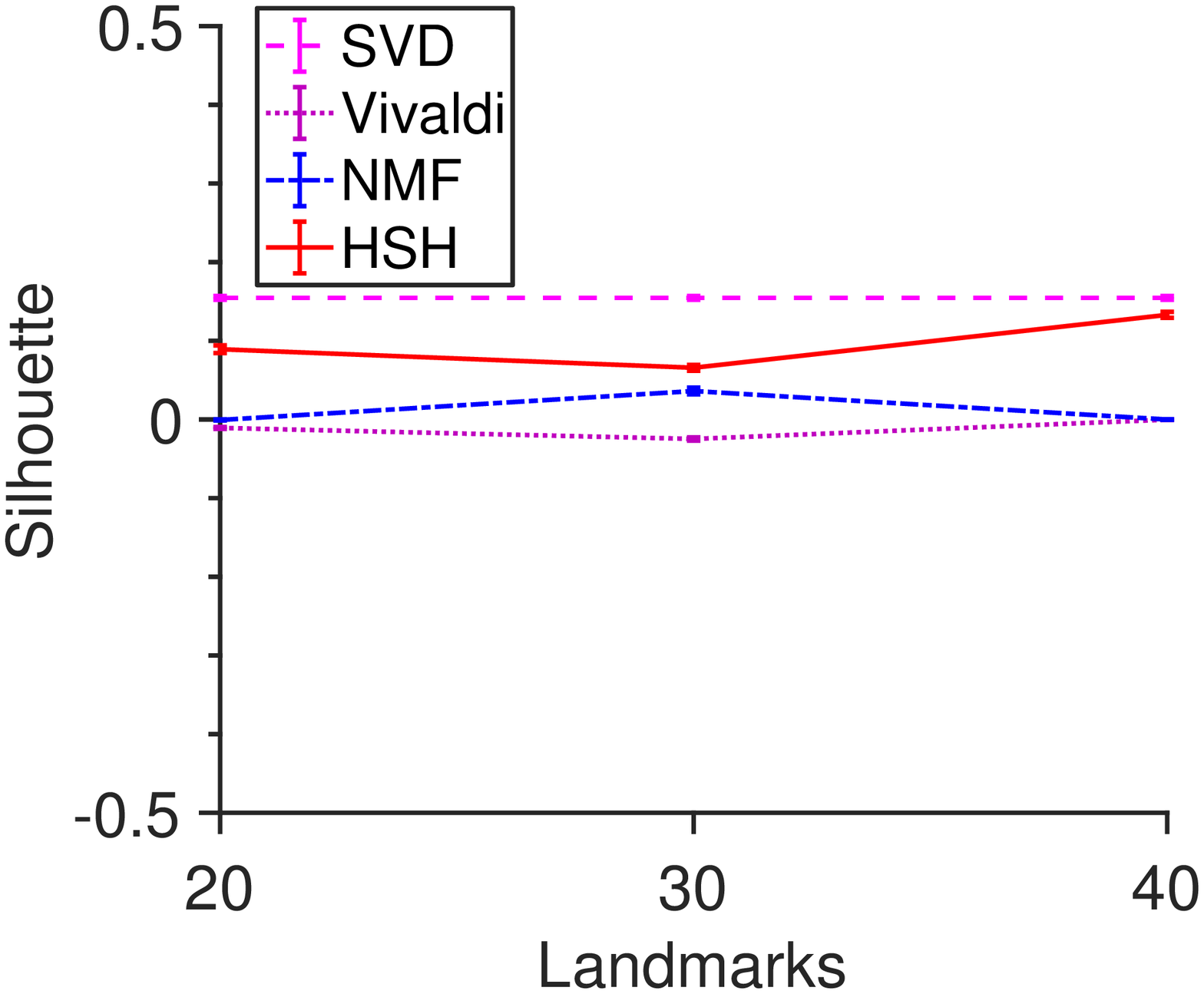, width=0.3\hsize}}
 %%\hspace{8pt}%
  \subfigure[DNS3997]{%
    \label{fig:ex4-b}%
   \epsfig{file=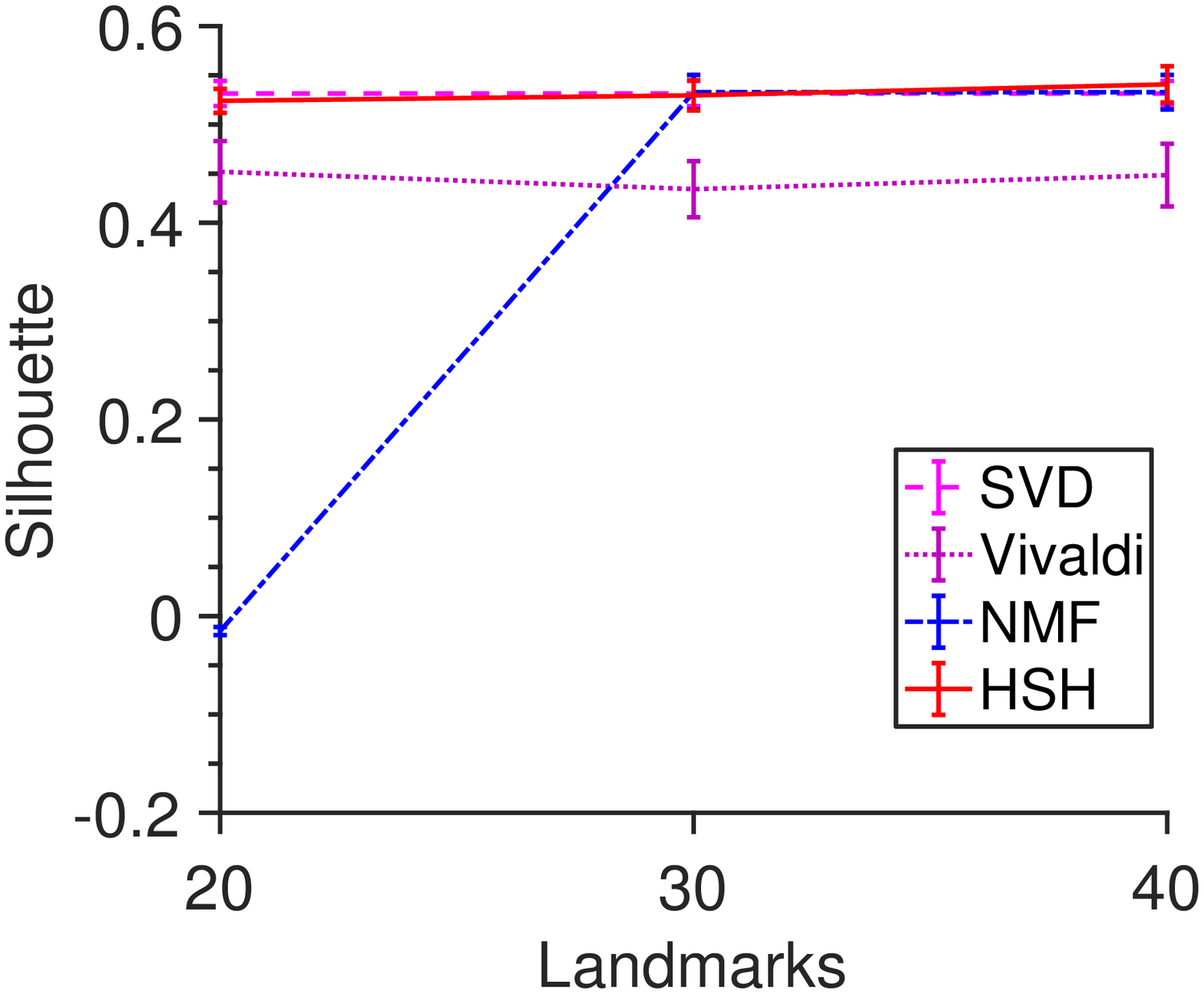, width=0.3\hsize}}
\caption{\label{fig:VLS}  Medians of the silhouette coefficients as well as the confidence intervals with different numbers of landmarks.} 
\end{figure*}

 \begin{figure*}
\centering
 \subfigure[DNS1143]{%
    \label{fig:ex4-a}%
    \epsfig{file=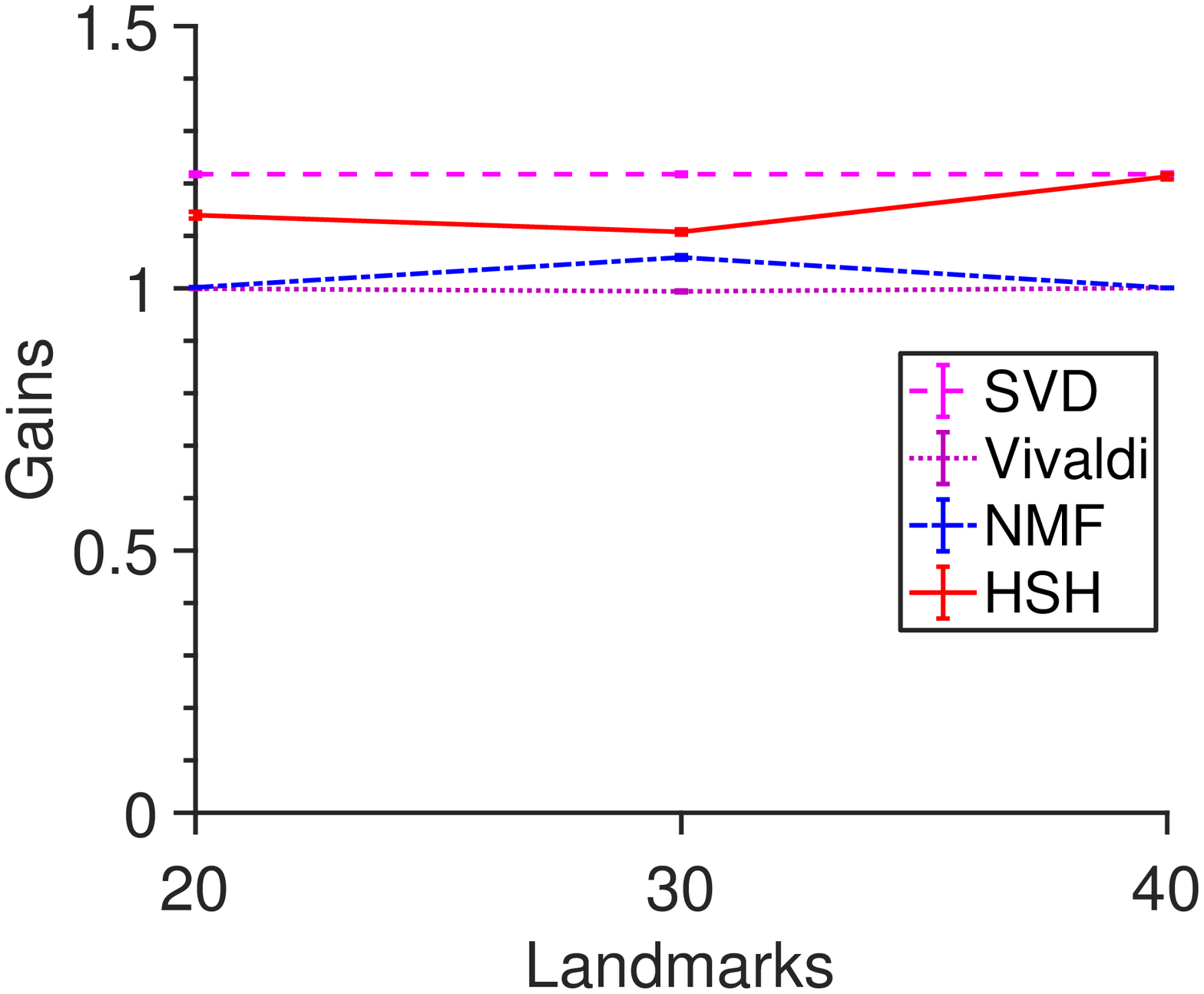, width=0.3\hsize}}
 %%\hspace{8pt}%
  \subfigure[DNS3997]{%
    \label{fig:ex4-b}%
   \epsfig{file=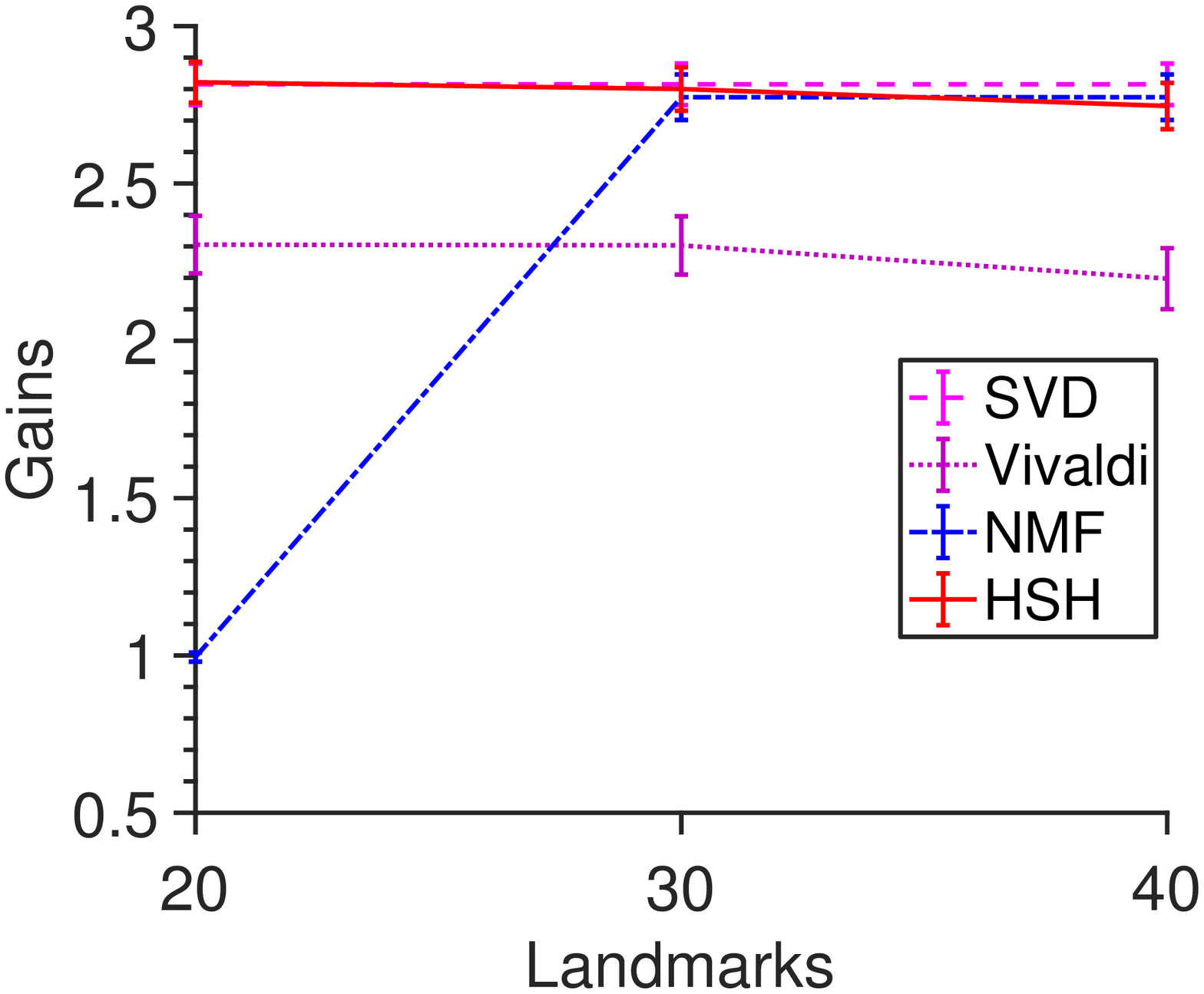, width=0.3\hsize}}
\caption{\label{fig:VLG}  Medians of the gain ratios  as well as the confidence intervals with different numbers of landmarks.} 
\end{figure*}

\co{
 \begin{figure*}
\centering
  \subfigure[$L$ = 20]{%
    \label{fig:ex4-b}%
   \epsfig{file=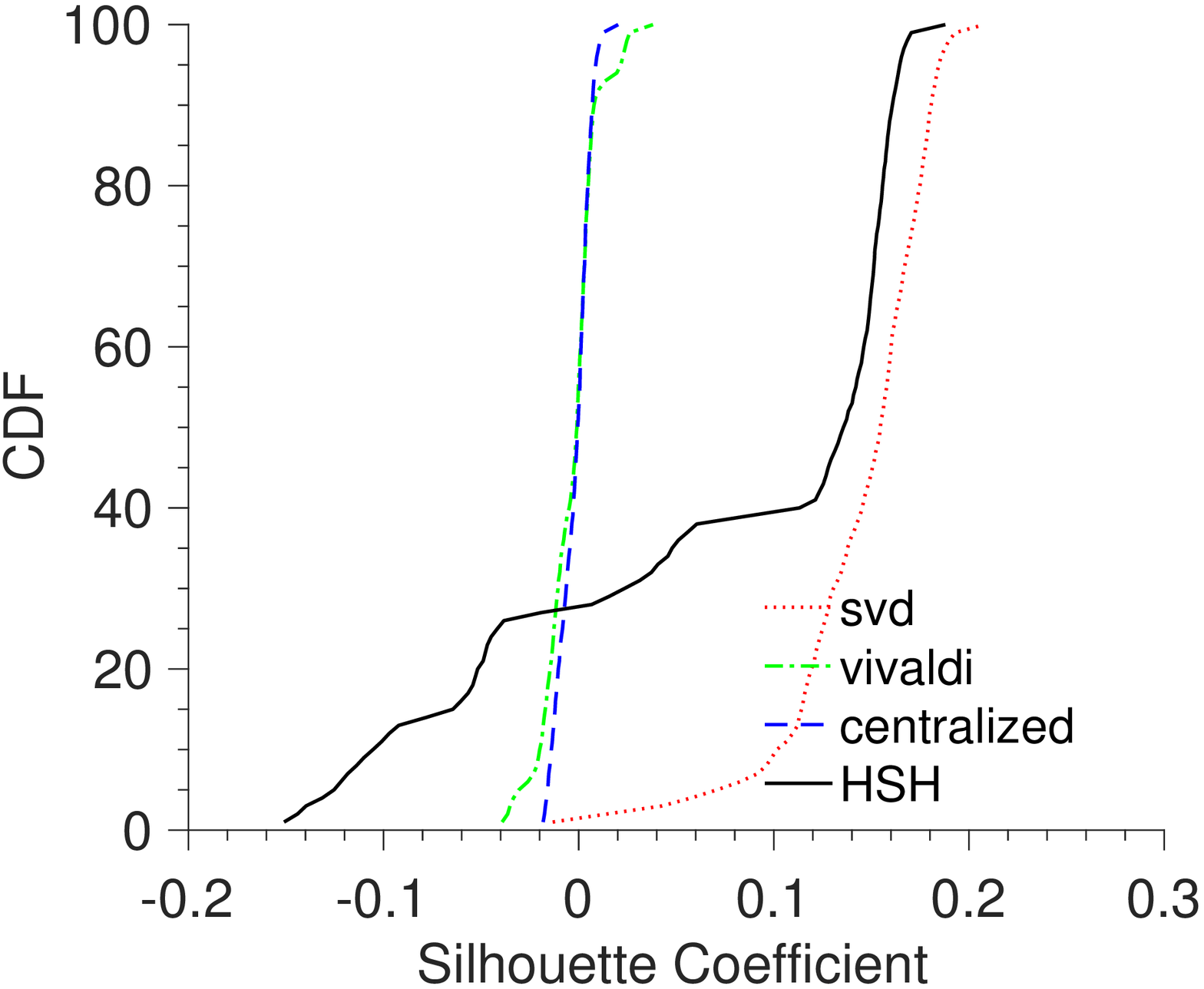, width=0.28\hsize}}
 \subfigure[$L$ = 30]{%
    \label{fig:ex4-a}%
    \epsfig{file=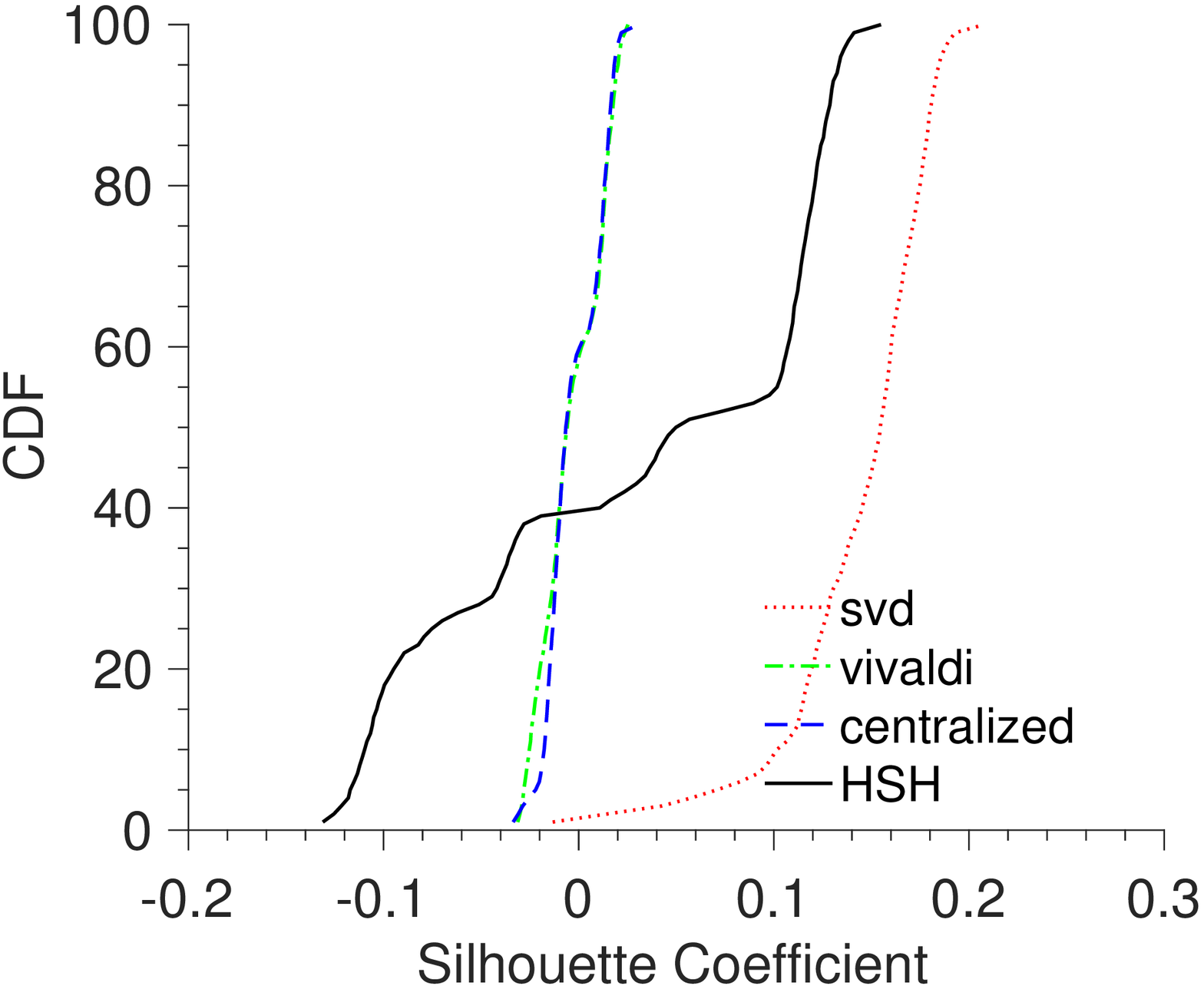, width=0.28\hsize}}
 %%\hspace{8pt}%
  \subfigure[$L$ = 35]{%
    \label{fig:ex4-b}%
   \epsfig{file=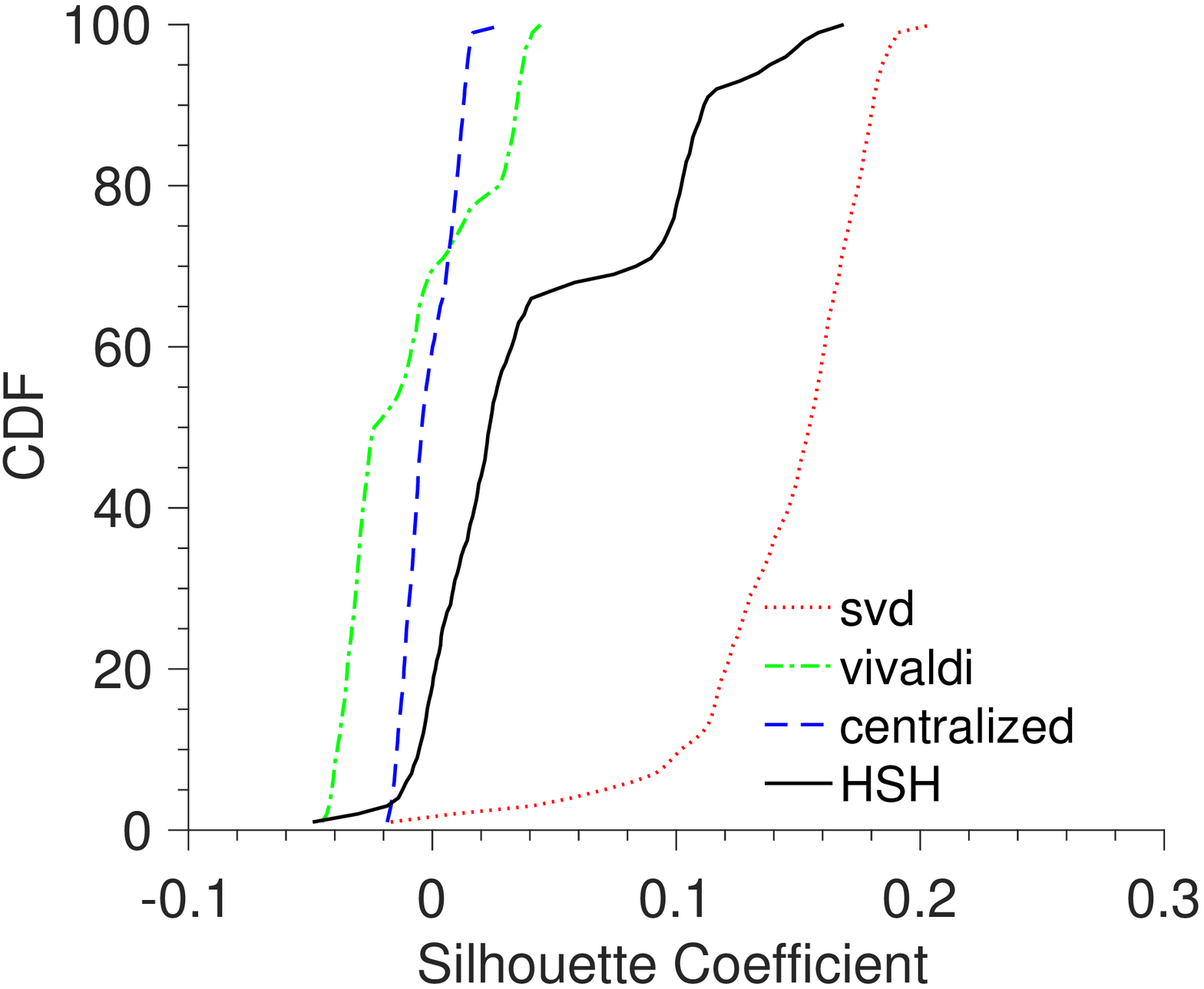, width=0.28\hsize}}
\caption{\label{fig:VLS}  Silhouette coefficients with different numbers of landmarks on the DNSDat data set.}
\end{figure*}

 \begin{figure*}
\centering
  \subfigure[$L$ = 20]{%
    \label{fig:ex4-b}%
   \epsfig{file=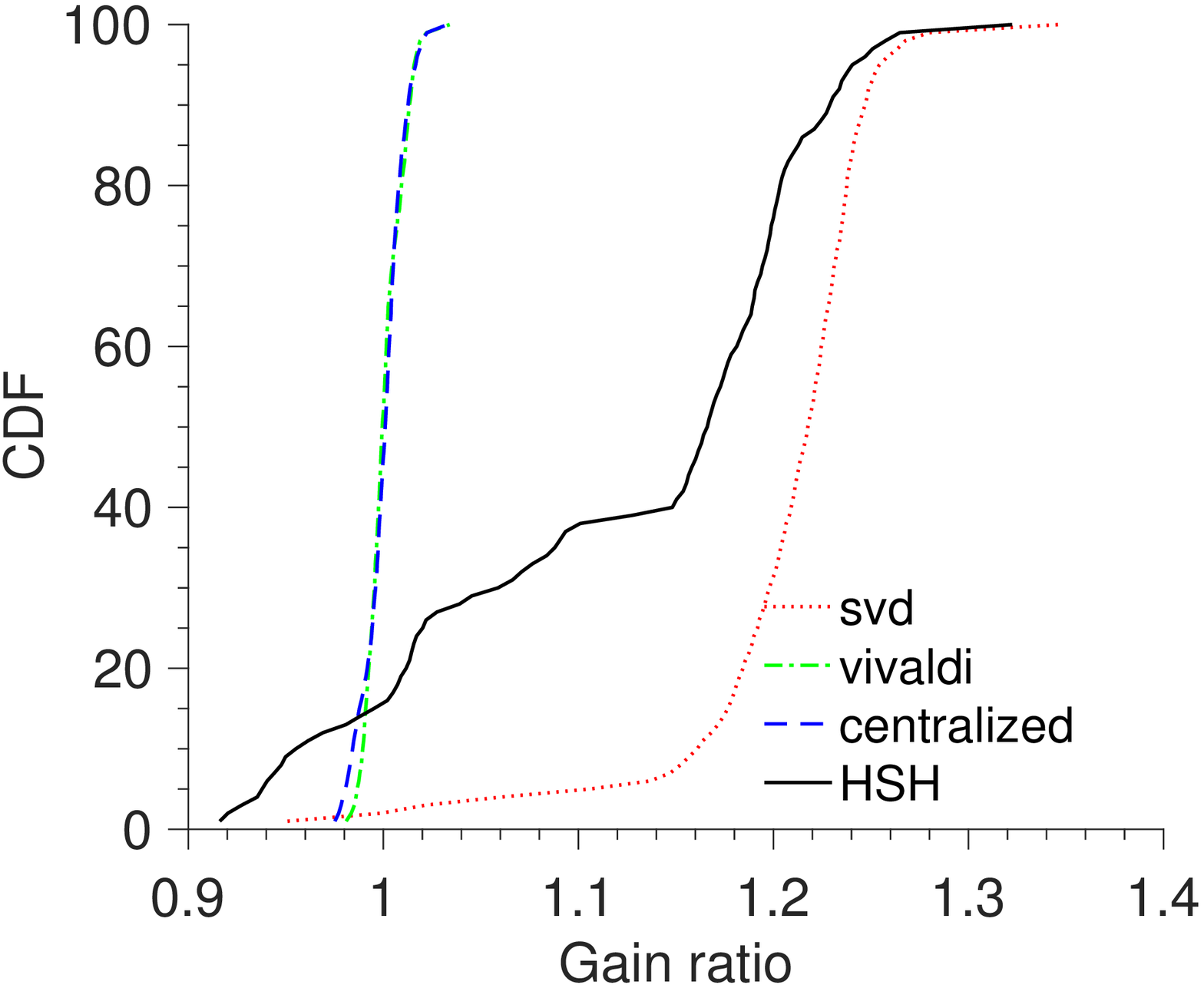, width=0.28\hsize}}
 \subfigure[$L$ = 30]{%
    \label{fig:ex4-a}%
    \epsfig{file=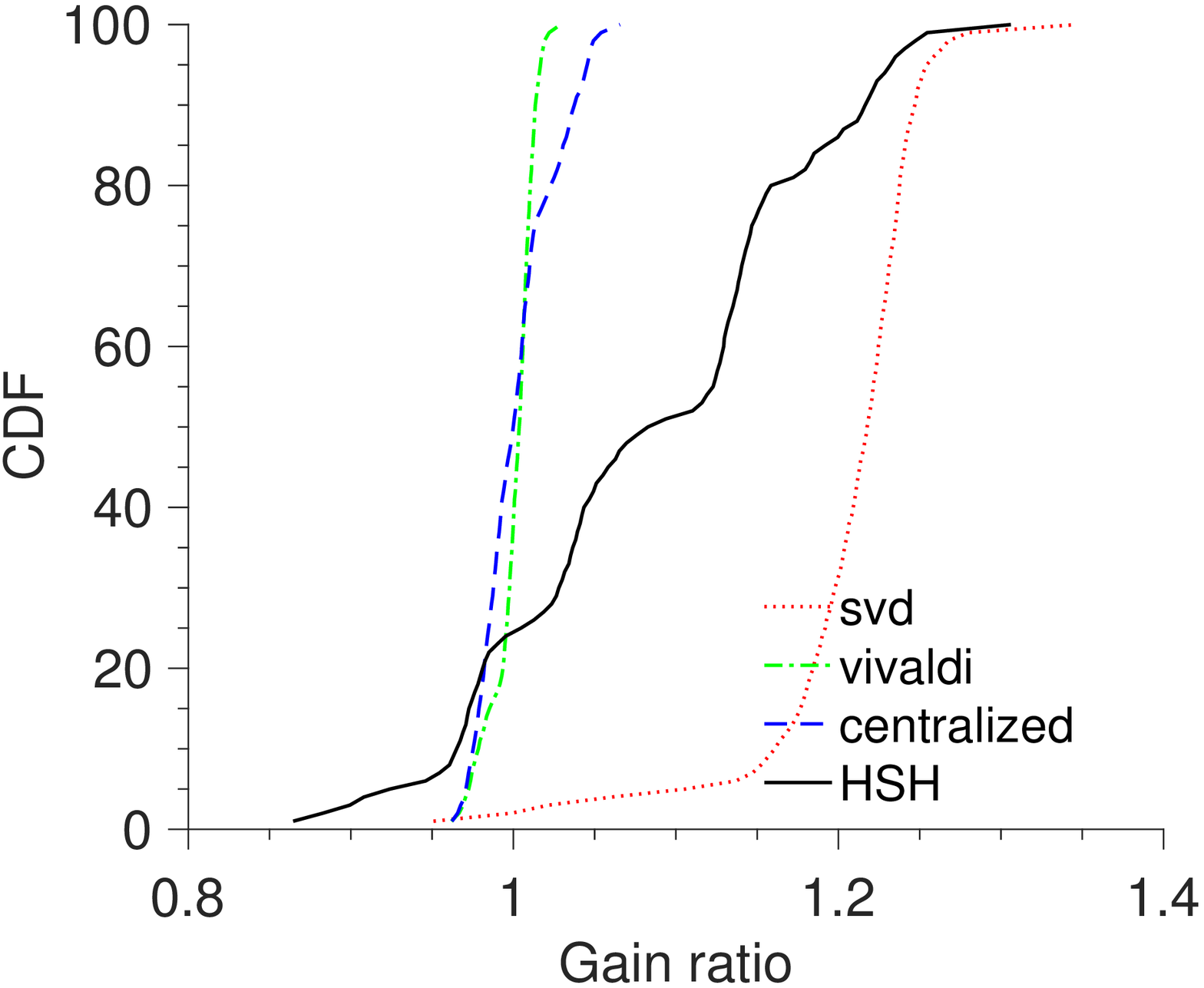, width=0.28\hsize}}
 %%\hspace{8pt}%
  \subfigure[$L$ = 35]{%
    \label{fig:ex4-b}%
   \epsfig{file=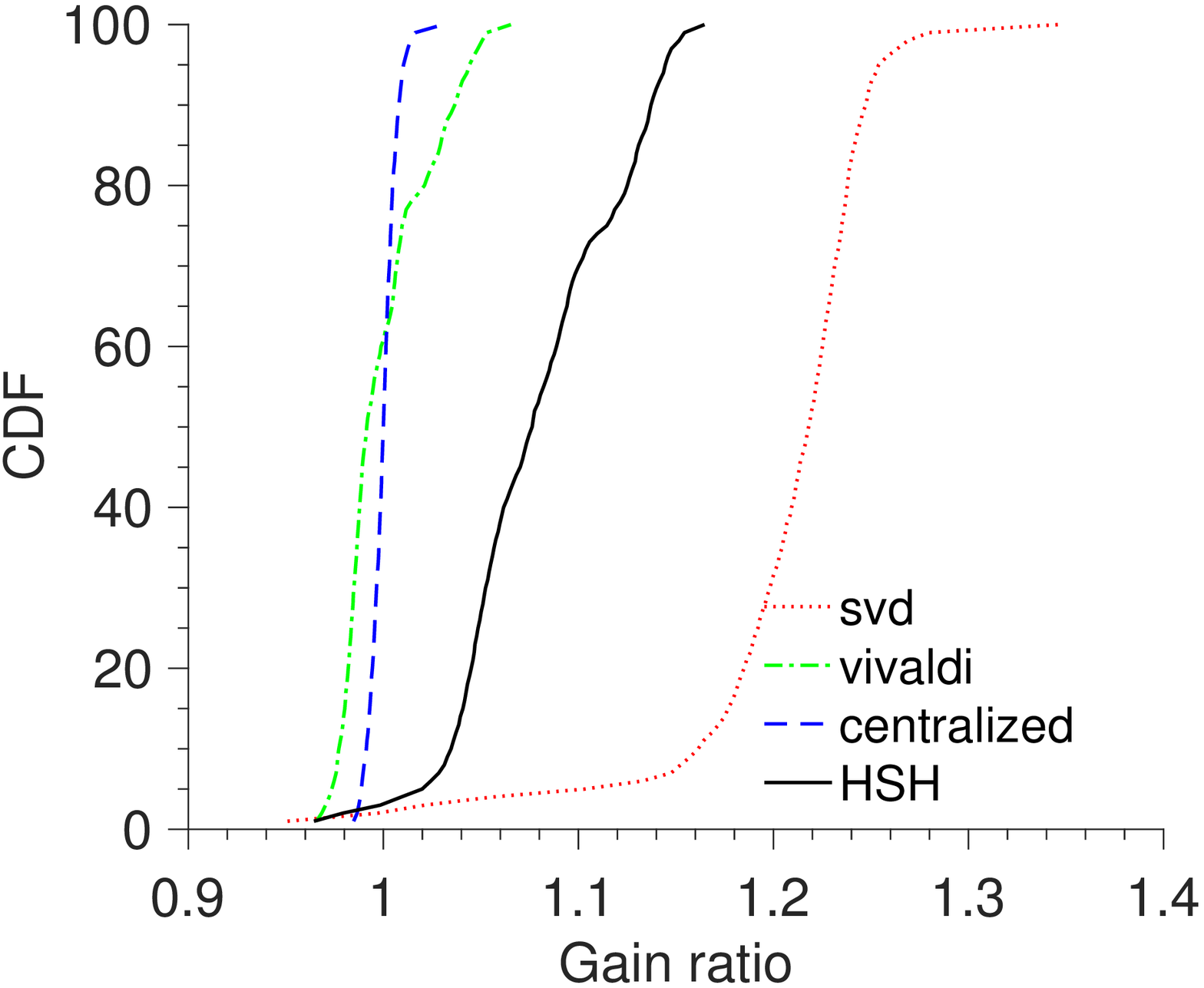, width=0.28\hsize}} 
\caption{\label{fig:VLG}  Gain ratios with different numbers of landmarks on the DNSDat data set.} 
\end{figure*}
}

\subsubsection{Cluster Number}

 \begin{figure*}
\centering
 \subfigure[DNS1143]{%
    \label{fig:ex4-a}%
    \epsfig{file=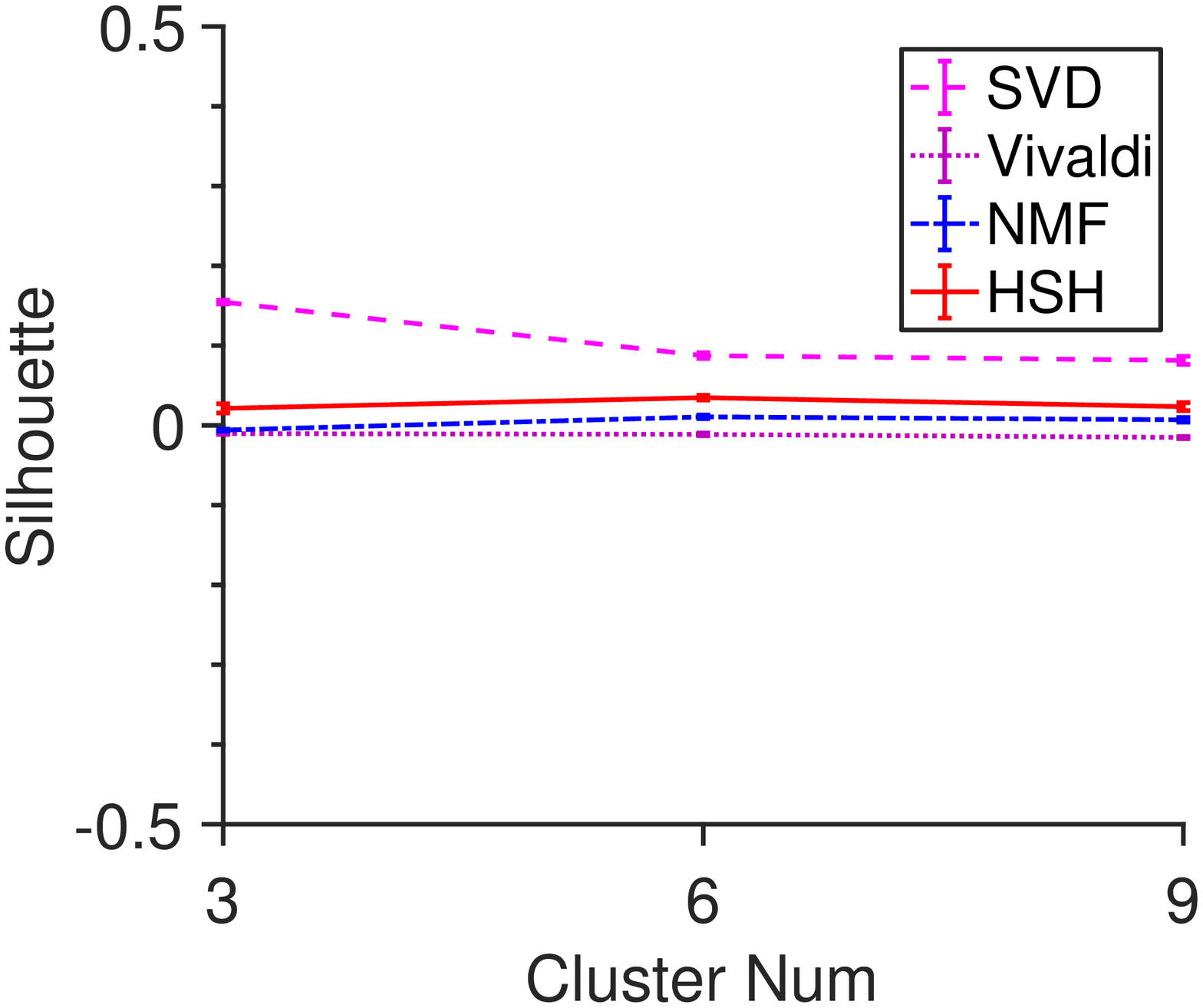, width=0.3\hsize}}
 %%\hspace{8pt}%
  \subfigure[DNS3997]{%
    \label{fig:ex4-b}%
   \epsfig{file=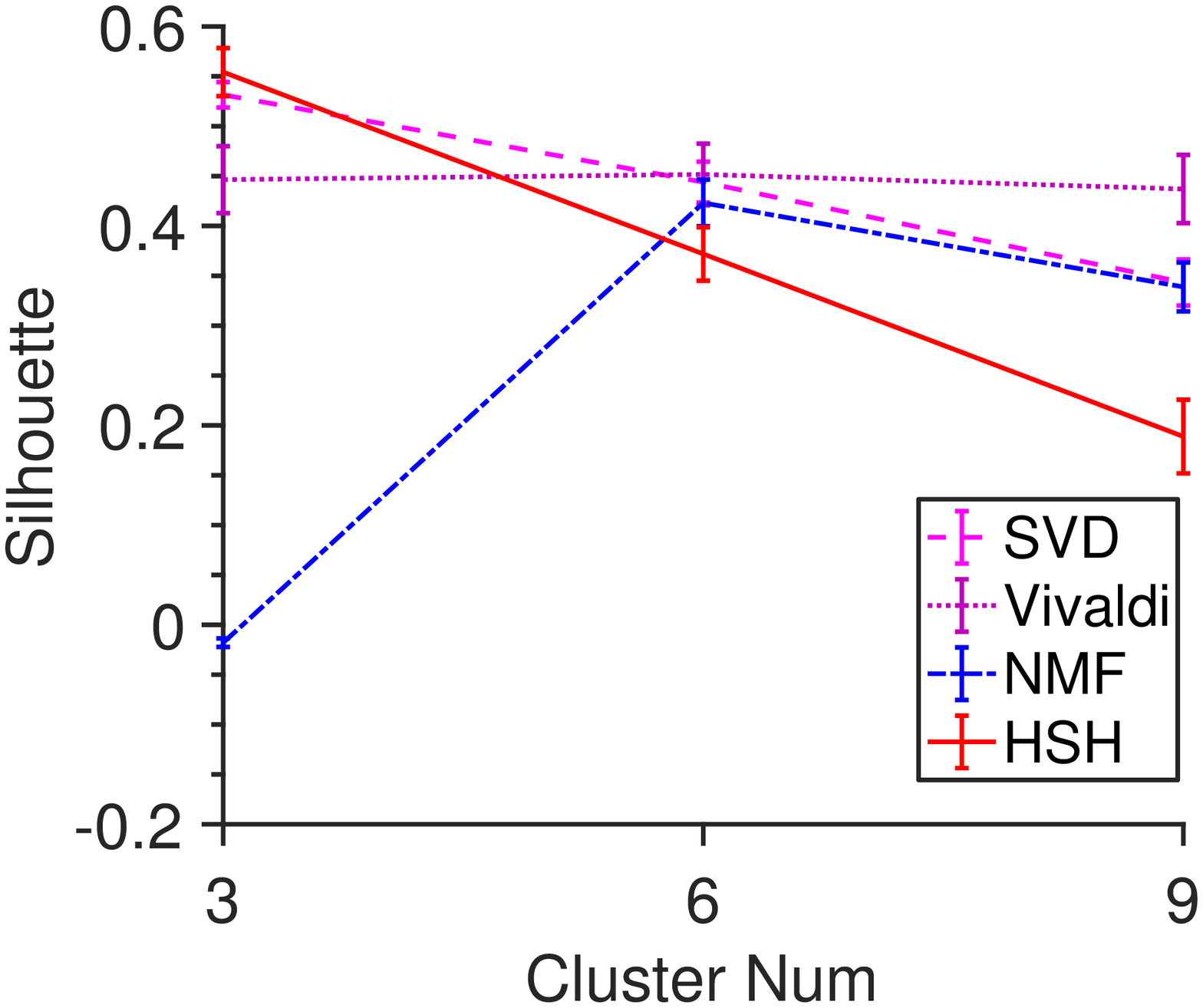, width=0.3\hsize}}
\caption{\label{fig:VCS}  Medians of the silhouette coefficients  as well as the confidence intervals as a function of the number of clusters.} 
\end{figure*}

 \begin{figure*}
\centering
 \subfigure[DNS1143]{%
    \label{fig:ex4-a}%
    \epsfig{file=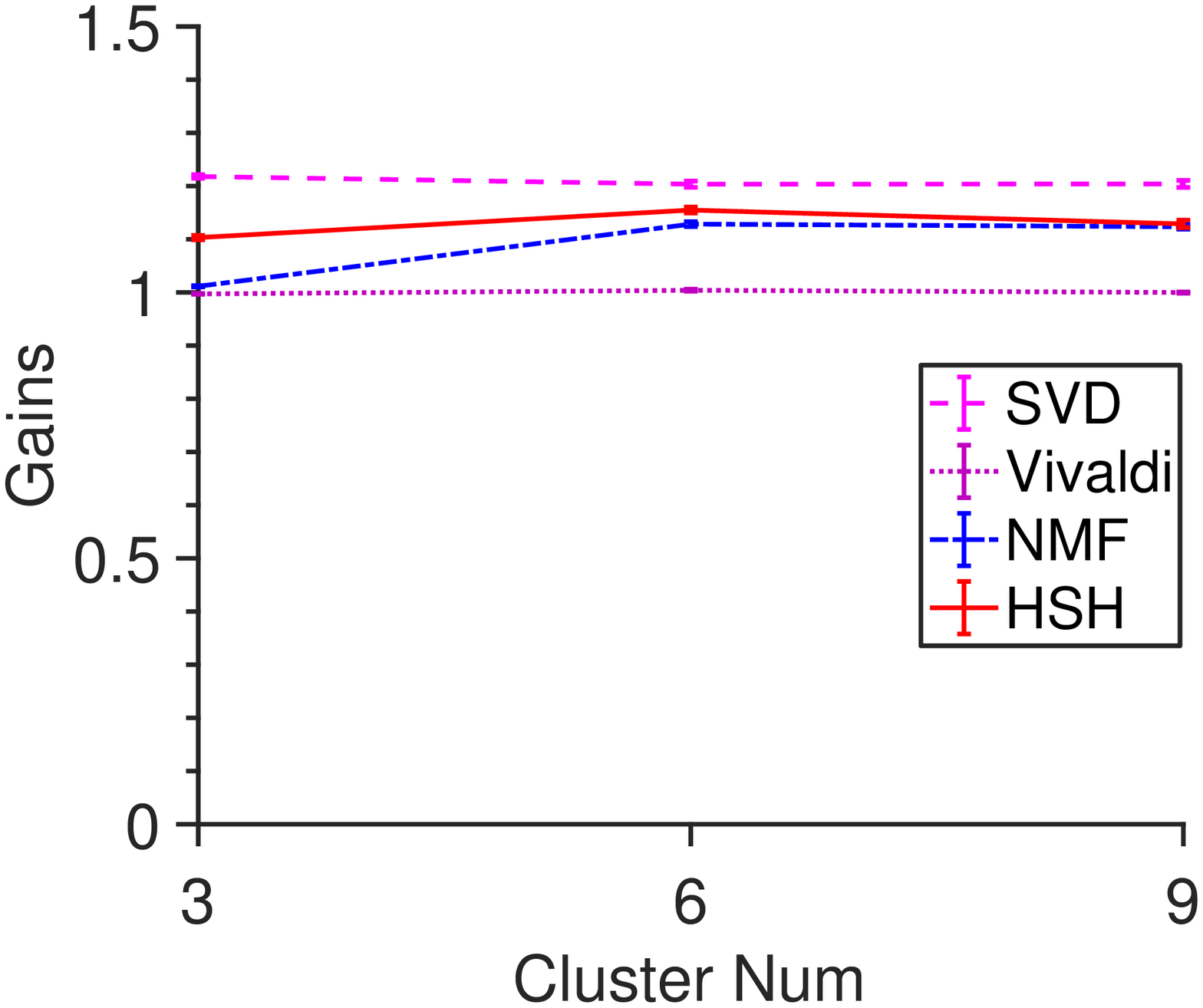, width=0.3\hsize}}
 %%\hspace{8pt}%
  \subfigure[DNS3997]{%
    \label{fig:ex4-b}%
   \epsfig{file=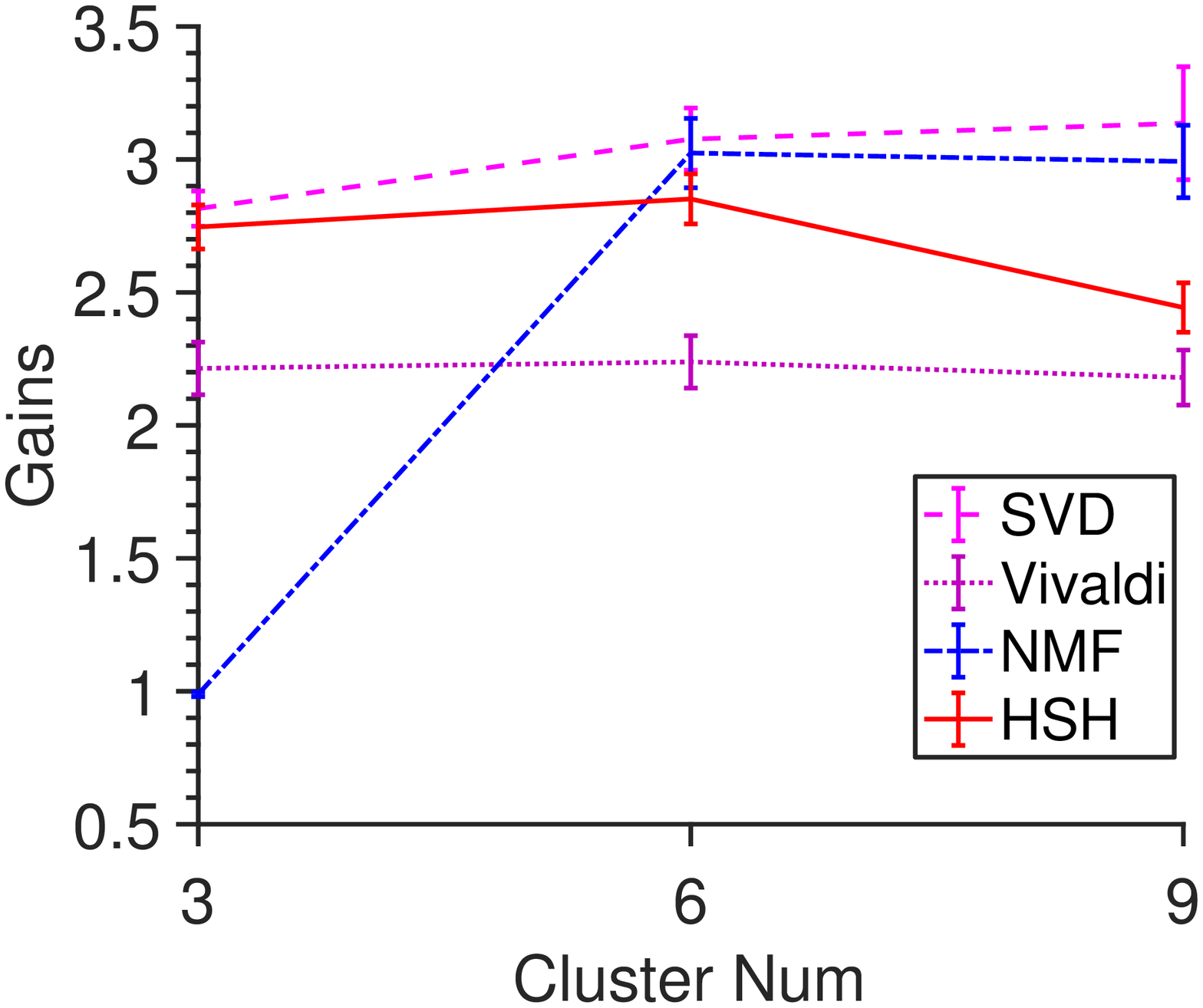, width=0.3\hsize}}
\caption{\label{fig:VCG}  Medians of the gain ratios as well as the confidence intervals as a function of the number of clusters.} 
\end{figure*}

Next, we fix the number of landmarks to 30 and test the clustering sensitivity to the number of clusters. The more sensitive to the numbers of clusters, the better the confidence on the numbers of clusters. Figure \ref{fig:VCS}  and \ref{fig:VCG}  plot the functions of  the silhouette coefficients and those of the gain ratios for four methods including SVD, vivaldi, centralized and HSH. We can see that the SVD and HSH are most sensitive to the variations of the numbers of clusters, while vivaldi and NMF are less sensitive.

\subsection{Dynamic Data Set}

Further, we test the dynamics of the clustering quality over the cloud data set, which contains 688 99$\times$99 pairwise RTT matrices between 99 wide-area network devices. We calculate the median of the silhouette coefficients and the gain ratios for each RTT matrix. We fix the number of clusters to three and the number of landmarks to 30 based on the above analysis.  
Figure \ref{fig:Dynamic} shows that both HSH and NMF have close clustering results and both 
have stable and high clustering quality.  HSH can find high
quality clustering structure stably. 

\begin{figure}[!t]
\centering
 \subfigure[Silhouette coefficient]{%
    \label{fig:ex3-a}%
    \epsfig{file=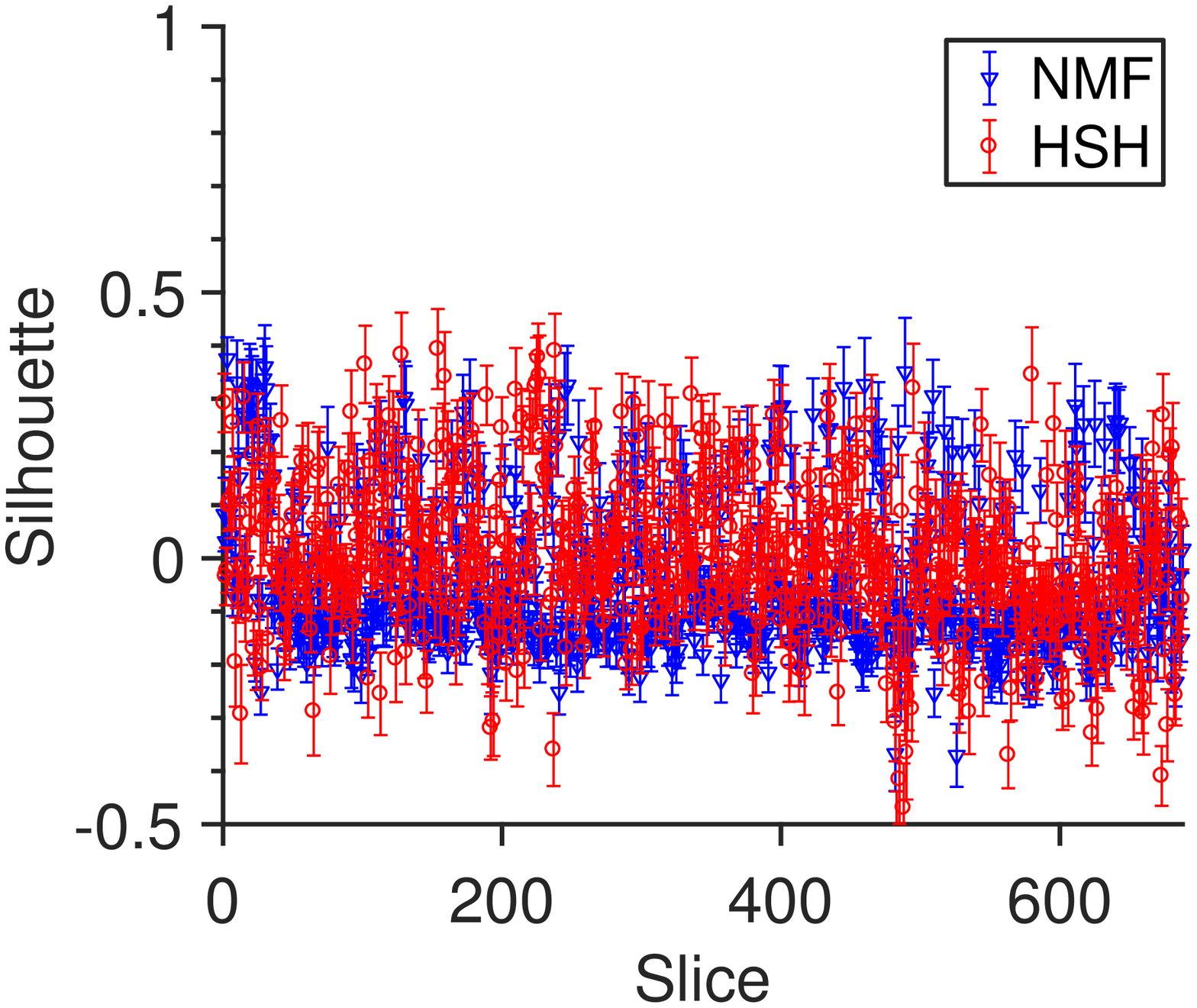,width=0.45\hsize}}
 % %\hspace{8pt}%
  \subfigure[Gain ratio]{%
    \label{fig:ex3-b}%
    \epsfig{file=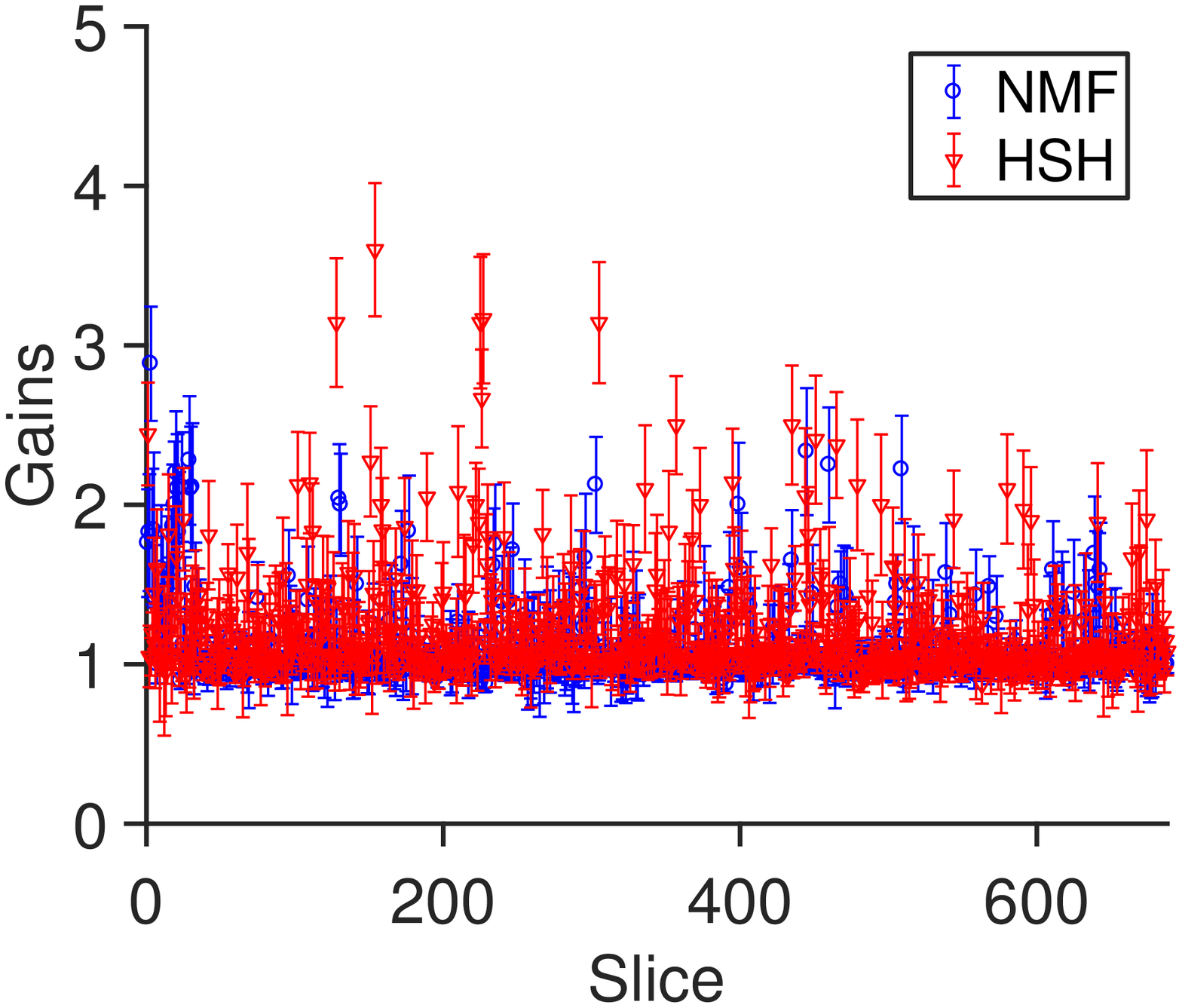,width=0.43\hsize}}
\caption[Clustering validation over the dynamic data set.]{\label{fig:Dynamic} Clustering validation over the dynamic data set.}
\end{figure}

\section{Discussions}
\label{ConcludeSec}

We have presented HSH, a new Nystr$\ddot{\textrm{o}}$m approximation approach for the kernel K-means clustering  framework on the complex-space kernel matrix.   We have validated the effectiveness of HSH over synthetic data sets,
and verified the clustering quality over  real-world DNS data sets.
The results show that HSH is scalable, accurate and robust.  

\newpage

\section{Broader Impact}

This paper is motivated to determine the clustering structure of global networked systems such as the DNS servers without direct control. Such global networked systems are the building blocks of the digital information society. Clustering networked systems based on network distances provides a compact summary representation. Understanding the structure of these systems could foster new technologies on service innovation to the global digital society.

\bibliographystyle{abbrv}
\bibliography{neurips_2020}

\appendix

\co{\section{Fundings}

This work was sponsored in part by National Key Research and Development Program of China under Grant No. 2018YFB0204300, and the National Natural Science Foundation of China (NSFC) under Grant No. 61972409, 61602500, 61402509, 61772541, 61872376.}

\end{document}